\documentclass[twoside,11pt]{article}

\usepackage{jmlr2e}

\usepackage{amssymb}
\usepackage{amsmath}
\usepackage{algorithm}
\usepackage[noend]{algpseudocode}
\usepackage{chngcntr}
\usepackage{mathtools}
\usepackage{bbm}
\usepackage{array}

\DeclarePairedDelimiter\ceil{\lceil}{\rceil}
\DeclarePairedDelimiter\floor{\lfloor}{\rfloor}
\DeclareMathOperator*{\argmin}{arg\,min}

\newcommand{\E}{\mathrm{E}}
\newcommand*{\QEDB}{\hfill\ensuremath{\blacksquare}}%

\usepackage{etoolbox}
\apptocmd{\appendices}{\apptocmd{\thesection}{: }{}{}}{}{}

\ShortHeadings{Adversarial Online Learning with noise}{Resler and Mansour}
\firstpageno{1}

\begin{document}

\title{Adversarial Online Learning with noise}

\author{\name Alon Resler \email alonress@gmail.com \\
       \addr Blavatnik School of Computer Science,
       Tel Aviv University,
       Tel Aviv, Israel
       \AND
       \name Yishay\ Mansour \email mansour.yishay@gmail.com \\
       \addr Blavatnik School of Computer Science, Tel Aviv University, 
       Tel Aviv, Israel\\
       and Google Research, Israel}

\editor{}

\maketitle

\begin{abstract}%   <- trailing '%' for backward compatibility of .sty file
We present and study models of adversarial online learning where the feedback observed by the learner is \textit{noisy}, and the feedback is either \textit{full information} feedback or \textit{bandit} feedback. Specifically, we consider binary losses xored with the noise, which is a Bernoulli random variable. We consider both a constant noise rate and a variable noise rate. Our main results are tight regret bounds for learning with noise in the adversarial online learning model.
\end{abstract}

\section{Introduction}

Online learning is a general framework for sequential decision-making under uncertainty. In each round, a  learner chooses an action from a set of $K$ available actions and suffers a loss associated with that action and observes ``some'' feedback about the losses. The losses in each round are arbitrary, possibly adversarial, and
the goal of the learner is to minimize the cumulative loss over a fix time horizon $T$. We measure the performance of the learner using the \textit{regret} which is the expected difference between the cumulative loss of the learner and that of the best fixed action.

Traditionally, there are two main types of feedback: \textit{full-information} feedback
and the \textit{Bandit} feedback.
In the \textit{full-information} feedback, often referred to as \textit{prediction with expert advice}, in each round the learner observes the losses of all actions. A typical example of the \textit{full-information} feedback is a hypothetical stock investor who invests all of his money in one of $K$ stocks on each day. At the end of the day, the investor incurs the outcome associated with the chosen stock and observes the outcomes of all the stocks. 
In the \textit{Bandit} feedback the learner only observes the loss associated with the action played. The typical example of the \textit{bandit} feedback is online advertising. Consider an Internet website that presents one of $K$ ads to each user, and its goal is to maximize the number of clicked ads. Naturally, we know whether the user clicked on the presented ad, but we have no information about other ads (whether the user would have clicked on them, if they were presented).

Both models have been extensively studied and 
received significant practical and theoretical interest.
The regret bound for the full information model is $\Theta(\sqrt{T \ln K})$ (see \cite{littlestone1994weighted,freund1997decision,kalai2005efficient}), and for the bandit model is $\Theta(\sqrt{KT})$. (See, \cite{auer2002nonstochastic,audibert2009minimax,cesa2006prediction,BubeckC12}). 
However, both models assumes that the observed feedback is exact.
In some real life scenarios, the feedback can be corrupted by noise, which is the focus of our work.
For example, in the web-advertising example, we might observe an incorrect feedback since a click event might be missing (due to network connection problem or logging error), recorded incorrectly (due to a browser issue, such as privacy setting), or alternatively, the user might be misidentified (due to multiple users using the same computer). 

In this paper, we present and study settings in which the feedback is corrupted by random noise. 
We assume that the losses are Boolean and that the noise is also Boolean, and the observation is the xor of the loss and the noise.
For the noise we consider Bernoulli random variable with probability $p$, denoted by $B(p)$. We consider a few variations of the noise model:

For the {\em constant noise rate}, we assume that there is a fixed probability $p$ for the noise (for all actions and rounds).  
For the  {\em variable noise rate}, we assume that there exists a distribution $D$ such that in each round $t$ we draw a vector of probabilities, where is $p_{i,t}$ the noise of action $i$ in round $t$. For both settings, we study both the case that the noise is known to the learner, and where it is unknown. Our main contribution is deriving tight regret bounds for those settings, both upper bounds (algorithms) and lower bounds (impossibility results). In the following we give a high level view of our results.

The {\em constant noise model} has a fixed parameter $\epsilon\in[0,1]$ and for every round $t$ the loss is xored with Bernoulli random variable with parameter $ p = \frac{1-\epsilon}{2}$. For the full information model we have a tight regret bound of $\Theta(\frac{1}{\epsilon}\sqrt{T \ln K})$, both when the noise parameter is known and when it is unknown. For the bandit feedback model we have a tight regret bound of $\tilde{\Theta}(\frac{1}{\epsilon}\sqrt{T K})$, both when the noise parameter is known and when it is unknown.

The {\em variable noise model} has a distribution $D$ over $[0,1]^K$ and at each round $t$, we draw from $D$ a realized noise vector $(\epsilon_{1,t},\ldots , \epsilon_{K,t})$, where $p_{i,t}= (1-\epsilon_{i,t})/2$ is the noise parameter for action $i$ at round $t$. 
In the following we describe our results for the
{\em uniform model}, where the marginal distribution of $D$ of each action is uniform $[0,1]$. For the full information we have a contrast between the case where the realized noise is observed, where we have a tight regret bound of $\Theta(T^{2/3} \ln^{1/3}K)$, and the case where the realized noise is not observed, where we have a linear regret, i.e., $\Theta(T)$. For the bandit model we have a tight bound of $\tilde{\Theta}(T^{2/3} K^{1/3})$, when the realized noise is not observed, and linear regret, i.e., $\Theta(T)$, when the realized noise is not observed. We also discuss the case of a general distribution and derive regret bounds for other specific distributions.
Our main results are summarized in Figure \ref{table}.

\begin{figure}[t]\label{table}
\begin{center}
\begin{tabular}{ | c | c| c |} 
%\begin{tabular}{ | m{7cm} | m{3cm}| m{3cm}  |} 
% \hline\\
% & &
% \\[-0.9em]
\textbf{Feedback type}  \textbackslash   \textbf{ Noise model}  & Constant noise& Variable noise  \\ 
& &(Uniform) \\ 
\hline
& &
\\[-0.9em]
Full information (known noise) & $\Theta(\frac{1}{\epsilon}\sqrt{T \ln K})$ & $\Theta(T^{2/3} \ln^{1/3}K)$ \\ 
& & 
\\[-0.9em]
\hline
& &
\\[-0.9em]
Full Information (unknown noise) & $\Theta(\frac{1}{\epsilon}\sqrt{T \ln K})$ & $\Theta(T)$ \\ 
& & 
\\[-0.9em]
\hline
& &
\\[-0.9em]
Bandit (known noise) & $\tilde{\Theta}(\frac{1}{\epsilon}\sqrt{T K})$ & $\tilde{\Theta}(T^{2/3} K^{1/3})$  \\
& & 
\\[-0.9em]
\hline
& &
\\[-0.9em]
Bandit (unknown noise) & $\tilde{\Theta}(\frac{1}{\epsilon}\sqrt{T K})$ & $\Theta(T)$  \\ 
\hline
\end{tabular}
\end{center}
\caption{Results summery}
\end{figure}

%\bigskip
%\noindent{\textbf{Related work:}}
\subsection*{Related work}

The work of \cite{kocak2016online} generalized a partial-feedback scheme proposed by \citep{mannor2011bandits,AlonCGMMS17}, in which the learner observes losses associated with a subset of actions which depends on the selected action, and considered a zero mean noise added to the side observations. Their main result is an algorithm that guarantees a regret of $\tilde{O}(\sqrt{ T})$ , where the constant depends on a graph property.

The work of \citet{wu2015online} studies a stochastic model where the feedback  of an action has the losses of each other action  with an additive noise of a zero-mean Gaussian, where variance depends both on the action played and observed. For this model they derive problem-depend lower bounds and matching upper bounds. 

\citet{gajane2018corrupt} studied a stochastic bandit problem where the feedback is drawn from a different distribution than the rewards, but there exist a link function relating them. They provide lower and upper bound for this setting.

%A more statistic approach of b
Binary sequence prediction with noise was studied by \citet{weissman2000universal} and \citet{weissman2001twofold}. They show upper bounds on the regret for binary sequence prediction with a constant noise rate (the binary sequence prediction model is implicitly a full feedback model). Their regret bound is similar to our regret bound (in the full information with constant noise).

There is a vast literature in statistics, operation research and
machine learning regarding various noise models. In
computational learning theory, popular noise models include
random classification noise \cite{AngluinL:1988} and malicious
noise \cite{Valiant:1985,KearnsL93}. The above noise models use the PAC model, and study the generalization error, while we consider an online setting and study the regret.

\bigskip
\textbf{Paper Organization:} Section \ref{sec:2} formalizes our model. Section \ref{sec:3} studies the \textit{full information with constant noise} settings, providing algorithms and matching lower bounds.
%for both setting, known and unknown noise. 
Section \ref{sec:4} studies \textit{full information with variable noise} settings and derives  algorithm, analyzes their regret, and proves a matching lower bound for specific noise distribution. 
Section \ref{sec:5} studies the \textit{bandit feedback} settings both for the constant noise and variable noise model.

\section{Model}\label{sec:2}

We consider adversarial decision problem with finite \textit{actions} (or \textit{actions}) set $A = \{1,2,\ldots, K\}$. On each round $t=1,2,\ldots,T$ the \textit{environment} selects a loss vector $\vec{\ell_{t}} \in \{0,1\}^K$ where $\ell_{i,t}$ is the loss associated with action $i$ at round $t$. Then, the \textit{learner} (or \textit{algorithm}) chooses an \textit{action} $I_t$ and incurs a loss $\ell_{I_t,t}$. 

The main difference between our models and the standard online model is that the learner observes a noisy feedback of the loss (to be specified separately in each setting).
Before presenting our models we start with a general definition of a noisy feedback of a single loss.

\begin{definition}
Let $\ell\in\{0,1\}$ be a loss, and let $\epsilon \in [0,1]$ be a parameter. We define the \textbf{$\epsilon$-noisy feedback} to be the following the random variable
\[
c = \ell \oplus R_{\epsilon}
\]
where $R_{\epsilon}$ is Bernoulli random variable with parameter $p=\frac{1-\epsilon}{2}$ (i.e., $\Pr[R_{\epsilon}=1]=p=\frac{1-\epsilon}{2}$). 
\end{definition}

Using the above definition we  present our four different settings, which are different in the feedback that the learner observes and the noise parameter selection.
The settings  are as follow:
\begin{enumerate}
\item \textbf{Full Information with Constant Noise:} In this setting, there exists a constant noise parameter $\epsilon\in[0,1]$, such that for every round $t$ the learner observes the \textit{$\epsilon$-noisy feedback}, $c_{i,t}$ for each action $i$, i.e., $c_{i,t}=\ell_{i,t}\oplus R_{\epsilon}$. 
%This setting divided into two settings base on whether the learner know the noise parameter $\epsilon$ or not.
\item\textbf{Full Information with Variable Noise:} In this setting, there exists a distribution $D$ over $[0,1]^K$. At the beginning of each round $t$, we draw from $D$ a realized noise $(\epsilon_{1,t},\ldots , \epsilon_{K,t})$, where $\epsilon_{i,t}\in[0,1]$ is the noise parameter action $i$ at round $t$. We assume that the noise vectors are drawn independently from $D$ at each round $t$. The learner observes, for each action $i$, an \textit{$\epsilon_{i,t}$-noisy feedback} $c_{i,t}$, i.e., $c_{i,t}=\ell_{i,t}\oplus R_{\epsilon_{i,t}}$. 
%This setting also divided into two settings base on whether the learner see the noise vector $(\epsilon_{1,t},\ldots , \epsilon_{K,t})$ or not.
\item \textbf{Bandit with Constant Noise:} In this setting, there exists a constant noise parameter $\epsilon\in[0,1]$, 
such that for every round $t$ the learner observes the \textit{$\epsilon$-noisy feedback}  of the action he played, i.e., $c_{I_t,t}= \ell_{i,t}\oplus R_{\epsilon}$ where $I_t$ is the action played in round $t$.  
\item\textbf{Bandit with Variable Noise:} In this setting, there exists a distribution $D$ over $[0,1]^K$. At the beginning of each round $t$, we draw from $D$ a realized noise $(\epsilon_{1,t},\ldots , \epsilon_{K,t})$, where $\epsilon_{i,t}\in[0,1]$ is the noise parameter action $i$ at round $t$. We assume that the noise vectors are drawn independently from $D$ at each round $t$.

The learner observes only the feedback for the action he played, i.e., $c_{I_t,t}= \ell_{I_t,t}\oplus R_{\epsilon_{I_t,t}}$, where $I_t$ is the action played in time $t$.
\end{enumerate}

Each of the models can have two variants: {\em known noise parameters}, where the learner observes the noise parameters or {\em unknown noise parameters}, where the learner doesn't observe the noise parameters. In the constant noise, the noise parameter is $\epsilon$ and in the variable noise, the noise parameters are the realized noise parameters at each round $t$, i.e., $(\epsilon_{1,t},\ldots , \epsilon_{K,t})$.
For our main results, we assume that the noise parameters are {\em known} to the learner.  When we examine the setting where the learner does not know the noise parameters, we state it explicitly.

We measure the performance of the learner using the (expected) \textit{regret}
of the {\em true losses}, namely,
\[
Regret(T) = \E\bigg[\sum_{t=1}^T \ell_{I_t,t}\bigg] - \min_{i \in A}\\E\bigg[\sum_{t=1}^T \ell_{i,t}\bigg]
\]
where the losses are selected by an adversary and the expectation is taken over the randomness of the algorithm and the randomness of the noise. 

\begin{algorithm}[t]
\caption{Exponential Weights Scheme}\label{alg:EWS}
\begin{algorithmic}[1]
\State \textbf{Initialization:} $w_{i,1} = 1$ for all $i \in A$
\State \textbf{Parameters:} $\eta > 0$
\For{$t = 1,2,...,T$}
\State Construct the probability distribution $q_t$ with $$q_{i,t} = \frac{w_{i,t}}{W_t} \text{ where } W_t=\sum_{i=1}^K w_{i,t}$$
\State Play a random action $I_t$ according to $q_t$
\State Incur loss $\ell_{I_t,t}$
\State Observe feedback according to the specific settings
\State Construct loss estimate $\hat{\ell}_{i,t} = EST(i, c_{i,t}, \vec{q}_t, I_t)$ for all $i \in A$
\State Update weights for all $i \in A$: $$w_{i,t+1} = w_{i,t}\exp(-\eta \hat{\ell}_{i,t})$$
\EndFor
\end{algorithmic}
\end{algorithm}

The algorithms presented in this paper are variants of the \textit{Exponential Weights Scheme} (see Algorithm \ref{alg:EWS}). In the \textit{Exponential Weights Scheme} (\textbf{EWS}) the algorithm maintains weight $w_{i,t}$ for each action $i$ (initially $w_{i,t} = 1$). On round $t$ the algorithm chooses an action proportional to the weights, based on a distribution $q_t$. After observing the feedback of round $t$, the algorithm updates the weights to $w_{i,t+1}$ using the previous weights $w_{i,t}$, the observations (i.e., $c_{i,t}$) and the noise parameter. 
Each noise setting determines how the feedback is constructed and observed (line 8 in the algorithm template). 
Each algorithm determines how to construct the loss estimate $\hat{\ell}_{i,t}$ (line 9 in the algorithm).
We denote generically $\hat{\ell}_{i,t} = EST(i, c_{i,t}, \vec{q}_t, I_t)$, where $EST$ is the loss estimate function that will be implemented differently in each setting and for each algorithm. 

\noindent{\bf Notations:} 
%We also present few notation that will help us later during the analysis: 
%We denote by $[K]=\{1,\ldots,K\}$ and 
Let $\hat{L}_{ON, T}=\sum_{t=1}^T \sum_{i=1}^K q_{i,t}\hat{\ell}_{i,t}$ and $\hat{L}_{k,T} = \sum_{t=1}^T \hat{\ell}_{k,t}$  the estimated loss of the online algorithm and of action $k$, respectively. We denote by $L_{ON, T} = \sum_{t=1}^T \sum_{i=1}^K q_{i,t}\ell_{i,t}$ and $L_{k,T} = \sum_{t=1}^T \ell_{k,t}$ the expected loss of the online algorithm and the loss of action $k$, respectively.
%For distributions we use the standard notations: 

We denote by $B(p)$ the \textit{Bernoulli distribution} with parameter $p$ and by $B(n,p)$ the \textit{Binomial distribution} with $n$ trials and parameter $p$.

\section{Full Information with Constant Noise model}\label{sec:3}

In this section, we consider the \textit{Full Information with Constant Noise} feedback model. In the first part, we derive an algorithm that uses the constant noise parameter $\epsilon$ and obtains regret bound of $O(\frac{1}{\epsilon}\sqrt{T \ln K})$. Then, we show how to obtain the same regret bound when the noise parameter $\epsilon$ is unknown. In the second part, we derive a lower bound, which shows that the regret of our algorithm is asymptotically optimal.     
% We start with the \textit{constant Bernoulli noise} lower and upper bound. Then study the \textit{Variable Noise} settings, show upper and lower bound for the case where $D=U(0,1)$ is the uniform distribution and then finish with a upper bound for general distribution. 

\subsection{Algorithms}

In this section we derive the algorithms that establish the upper bound on the regret. The idea is to construct an unbiased estimator for each loss. Let $\epsilon \in [0,1]$ and let $p = \frac{1-\epsilon}{2}$ be the noise parameter. The unbiased estimator is
\[
EST(i, c_{i,t}, \vec{q}_t, I_t) = \frac{c_{i,t} - p}{1-2p}=\hat{\ell}_{i,t}\;.
\]
The estimator is unbiased since,
$$
\E[\hat{\ell}_{i,t}] = \frac{p(1-\ell_{i,t}) + (1-p)\ell_{i,t} - p}{1-2p} = \ell_{i,t}\;.
$$

The following theorem establishes the regret bound when we use the 
Exponential Weights Scheme with the above unbiased estimator.

\begin{theorem}\label{thm:2}
Let $\epsilon \in [0,1]$, denote $p = \frac{1-\epsilon}{2}$ and assume $T \geq \frac{1}{4}\ln K$. Then running \textit{Exponential Weights Scheme} under the \textit{Full Information with Constant Noise} setting with the following loss estimate
\[
EST(i, c_{i,t}, q_t, I_t) = \frac{c_{i,t} - p}{1-2p} = \hat{\ell}_{i,t}
\]
and for $\eta = \epsilon \; \sqrt{\frac{\ln K}{T}}$ we have,
$$
Regret(T) \leq \frac{2}{\epsilon}\sqrt[]{T \ln K}
$$
\end{theorem}

%We start with lemma which is building block for all subsequent results 
The following lemma establishes a well known property of EWS, and for completeness we give its proof in Appendix \ref{app:lem:3}.

\begin{lemma}\label{lem:3}
Let $\eta > 0$ and a sequence of loss estimates $\hat{\ell}_1,\ldots,\hat{\ell}_T$ where $t\quad \hat{\ell}_t : \{1,\ldots,K\} \rightarrow \mathbb{R}$ such that  $-\eta\hat{\ell}_{i,t} \leq 1$ for all $i$  and $t$, then the probability vectors $\vec{q}_1,\ldots,\vec{q}_T$ define in the \textit{Exponential Weights Scheme}, for any action $k$, satisfies
\begin{equation*}
\sum_{t=1}^T \sum_{i=1}^K q_{i,t}\hat{\ell}_{i,t} -  \sum_{t=1}^T \hat{\ell}_{k,t} \leq \frac{\ln K}{\eta} +
\eta\sum_{t=1}^T \sum_{i=1}^K q_{i,t} (\hat{\ell}_{i,t})^2
\end{equation*}
\end{lemma}
\noindent

Full proof of Theorem \ref{thm:2} is given in Appendix \ref{app:thm:2}, and follows  by using Lemma \ref{lem:3}, the fact that the estimator is unbiased, and bounding the second moment of the estimator by $\E[(\hat{\ell}_{i,t})^2] \leq \frac{1}{\epsilon^2}$.

In  Theorem~\ref{thm:2}, the learner uses the noise parameter $\epsilon$ to derive an unbiased estimator. The following theorem (proof in Appendix \ref{app:thm:4}) shows that the same regret bound can be attained even when the leaner does not know the noise parameter $\epsilon$.

\begin{theorem}\label{thm:4}
Let $\epsilon \in [0,1]$ and denote $p = \frac{1-\epsilon}{2}$. Running \textit{Exponential Weights Scheme} under the \textit{Full Information with Constant Noise} setting with the following loss estimate
\[
EST(i, c_{i,t}, \vec{q}_t, I_t) = c_{i,t}=\hat{\ell}_{i,t}
\]
and for $\eta = \epsilon \; \sqrt{\frac{\ln K}{T}}$, we have,
$$
Regret(T) \leq \frac{2}{\epsilon}\sqrt{T \ln K}
$$
\end{theorem}

\subsection{Impossibility result}

In this section we derive a lower bound on the regret for the Full Information with Constant Noise model.
Our lower bound matches our upper bound, up to a constant factor.
Specifically, the following theorem gives us a lower bound of $\Omega(\frac{1}{\epsilon}\sqrt{T\ln K})$ on the regret.
\begin{theorem}
\label{thm:5}
Consider the \textit{Full Information with Constant Noise} setting with noise parameter $\epsilon \in (0, \frac{1}{2})$, $T \geq 2 \ln K$ and $K \geq e^{18}$. Then for any algorithm, there exists a sequence of loss vectors $\vec{\ell}_1,...,\vec{\ell}_T$ such that 
\[
Regret(T) = \Omega(\min \{\frac{1}{\epsilon} \; \sqrt[]{T \ln K}, T\})
\]
\end{theorem}

The proof idea is to define a stochastic strategy for loss assignment, which is a distribution over problem instances. Then, by showing that any algorithm suffers high expected regret, where the expectation is over the problem instances defined by the strategy, conclude that there exists a problem instance with high regret. (The proof is given in Appendix \ref{app:thm:5}.)

\section{Full Information with Variable Noise model}\label{sec:4}

In this section we investigate the \textit{Full Information with Variable Noise} settings. 
Recall that in this setting we have a distribution $D$ over $[0,1]^K$. At the beginning of each round $t$, we draw from $D$ a realized noise $(\epsilon_{1,t},\ldots , \epsilon_{K,t})$, where $\epsilon_{i,t}$ is the noise parameter for action $i$ at round $t$. We assume that the noise vectors are drawn independently from $D$ at each round $t$ (however, there can be correlations between the noise parameters $\epsilon_{i,t}$ of different actions at the same round $t$).  The learner observes the realized noise $(\epsilon_{1,t},\ldots , \epsilon_{K,t})$ and then picks an \textit{action} $I_t\in A$. Then, the learner observes the \textit{$\epsilon_{i,t}$-noisy feedback} $c_{i,t}$ for each action $i$. We denote by $p_{i,t} = \frac{1-\epsilon_{i,t}}{2}$.

The section is structured as follows.
Initially, we investigate the case of a uniform distribution over 
$[0,1]$, that is, the marginal distribution of $D$ for each action $i$ is uniform over $[0,1]$, i.e., 
 $\epsilon_{i,t} \sim U(0,1)$, where $U(0,1)$ is the uniform distribution on $[0,1]$. Following that, we generalize the regret bound for a general noise distribution $D$. We conclude with a few examples of specific distributions. 

\subsection{\textbf{Uniform Noise Distribution}}

\subsubsection{Algorithm}

A simple potential approach to the problem is to try to use the \textit{Exponential Weights Scheme} with the unbiased estimator
\[
EST(i, c_{i,t}, \vec{q}_t, I_t) = \frac{c_{i,t} - p_{i,t}}{1-2p_{i,t}}
\]
as in the constant noise settings. 
A close examination reveals that there is a problem when $p_{i,t}$ is close to $1/2$ (i.e., $\epsilon_{i,t}$ is close to $0$). In such cases the estimator is unbounded and can give a very high value. An intuitive idea is to avoid using feedbacks with high noise. This is implemented by the learner by having an additional parameter $\theta$ and ignoring feedbacks where $p_{i,t} > \frac{1-\theta}{2}$ (i.e., $\epsilon_{i,t} < \theta$). More formally, we use the \textit{Exponential Weights Scheme} with the following estimator:
\[
EST(i, c_{i,t}, \vec{q}_t, I_t) = \frac{c_{t,i}-p_{i,t}}{1-2p_{i,t}} \mathbbm{1}_{\{p_{i,t} \leq \frac{1-\theta}{2}\}} =\hat{\ell}_{i,t}
\]
The algorithm resulting from using the above estimator in the \textit{Exponential Weights Scheme} is called \textbf{EW-Threshold}. We prove the following regret bound in Appendix \ref{app:thm:10} 

\begin{theorem}
\label{thm:10}
Let $D$ be the noise distribution, such that for each action $i$ the marginal distribution $\epsilon_{i,t}$ is distributed $U(0,1)$ (but not necessarily independent for different actions). The \textbf{EW-Threshold} algorithm with the parameters
\[
\eta = (\frac{\ln K}{T})^{2/3} \text{ and } \theta = (\frac{\ln K}{T})^{1/3}
\]
has, in the \textit{Full Information with Variable Noise} setting, a regret of at most,
$$
Regret(T) \leq 3 T^{2/3} (\ln K)^{1/3}
$$
\end{theorem}

\subsubsection{Impossibility Result}

In this section we derive a lower bound on the regret of $\Omega(T^{2/3}(\ln K)^{1/3})$. Together with the upper bound we obtain that for the \textit{Full Information with Variable Noise} we have
\begin{equation*}
Regret(T) = \Theta(T^{2/3}(\ln K)^{1/3})
\end{equation*}

For the lower bound we use a specific noise distribution $D$, denoted by $D'$.
In $D'$, all the noise of individual actions are identical, and uniformly distributed.
Formally, we generate the noise parameters from $D'$ as follows. We draw $\epsilon_t \sim U(0,1)$ and for every $i$ we set $\epsilon_{i,t}=\epsilon_t$. 

The idea behind the proof (see Appendix \ref{app:thm:11} for formal proof) is to use adversarial strategy for loss assignment in the following way: when the noise is \textit{low}, all the actions will have the same loss, but when the noise is \textit{high}, one action, chosen randomly at the beginning, will be superior.

\begin{theorem}\label{thm:11}
Any algorithm in the \textit{Full Information with Variable Noise} setting with the noise distribution $D'$, there exist a series of loss vectors $\vec{\ell}_1,...,\vec{\ell}_T$ such that 
\[
Regret(T) = \Omega(T^{2/3}(\ln K)^{1/3})
\]
\end{theorem}

\subsection{General distributions}

In this section we generalized the result of \textbf{EW-Threshold} to a general noise distribution $D$. We assume that the marginal distribution of each action $i$ is the same and we denoting the CDF (Cumulative Distribution Function) of it by $F$. 
Then, we use our generalized bound to derive a sub-linear regret upper bound for distributions $D$ that satisfies a given condition. 
We extend the proof of Theorem \ref{thm:5} and obtain the following general upper bound (proof in Appendix~\ref{app:thm:12}).

\begin{theorem}\label{thm:12}
Let $D$ be a distribution, such that the marginal distribution over $\epsilon_{i,t}\in (0,1)$ has a CDF $F$. Then, running \textbf{EW-Threshold}  algorithm with parameters $\theta > 0 \text{ and } \eta > 0$ satisfies,
\[
Regret(T) \leq  \frac{\ln K}{\eta} + \eta T g(\theta) + F(\theta)T
\]
where $g(\theta)=\E[\frac{1}{\epsilon^{2}}\mathbbm{1}_{\epsilon \geq \theta}]$. Moreover, for $\eta = \sqrt{\frac{\ln K}{Tg(\theta)}}$ we have 
\begin{equation*}
Regret(T) \leq 2\sqrt{g(\theta)T \ln K} + F(\theta)T
\end{equation*}
\end{theorem}

The following corollary (proof in Appendix \ref{app:cor:13}) gives a general upper bound that depends only on a property of the noise distribution $D$. We assume that all the marginal distributions of $D$ are identical and with CDF $F$. 

\begin{corollary}\label{cor:13}
Let $D$ be a noise distribution, where each marginal distribution has the same CDF $F$, and assume $F(\theta) \leq \theta^\alpha$  for a given $\alpha > 0$. Then 
\begin{equation*}
Regret(T) = O(T^{\frac{2+\alpha}{2+2\alpha}} (ln K)^{\frac{\alpha}{2(1+\alpha)}})
\end{equation*}
\end{corollary}

To get an intuition for the bound of Corollary \ref{cor:13} we can consider a few intuitive settings of the parameter $\alpha$. 
The uniform distribution has $\alpha=1$, and the theorem yield $Regret(T) = \tilde{O}(T^\frac{3}{4})$, which is higher than the regret bound computed explicitly in Theorem~\ref{thm:10}, of $O(T^{2/3})$. When $\alpha \to \infty$ we have $Regret(T) \to \tilde{O}(T^{\frac{1}{2}})$, which is tight even without any noise. When $\alpha \to 0$, we have no restriction on the noise distribution, and indeed the theorem yields a linear regret bound.

To give an example for the bound of Theorem \ref{thm:12} we prove a regret bound for truncated exponential distribution (proof in Appendix \ref{app:cor:14}).

\begin{corollary}\label{cor:14}
Let $D$ be a distribution, such that the marginal distribution over $\epsilon_{i,t}\in (0,1)$ has a PDF $f(x)=\frac{\lambda}{1-e^{-\lambda}}e^{-\lambda x} \text{ for } x \in (0,1) \text{ and } \lambda > 0$. Then, running \textbf{EW-Threshold}  algorithm with parameters $\theta > 0 \text{ and } \eta > 0$ satisfies,
\[
Regret(T) \leq  3\lambda T^{2/3}(\ln K)^{1/3}
\]
\end{corollary}

%%%%%%%%%%%%%%%%%%%%%%%%%%%%%%%%%%%%% 

\subsection{Importance of knowing the Noise}

Until now we assume for the \textit{Full Information with Variable Noise} that the learner observes the noise drawn for each action $p_{i,t}$, before picking an action. In the \textit{Full Information with Constant Noise} we showed that this information is not critical and the same regret bound can be achieved without this information. The following theorem states that in the \textit{Full Information with Variable Noise}, a learner cannot achieve sub-linear regret without observing the noise.

\begin{theorem}\label{thm:15}
Fix an algorithm for the  \textbf{Full Information with Variable Noise} model under the uniform marginal distribution and assume that in each round $t$ the learner does not observes the noise parameters $(\epsilon_{1,t},...,\epsilon_{K,t})$.
%before picking an action $I_t$, 
Then, there exist a sequence of loss vectors $\vec{\ell}_1,...,\vec{\ell}_T$ such that 
\[
Regret(T) = \Omega(T)
\]
\end{theorem}

We prove the theorem for the case of $K=2$. 
The idea behind the proof is to use stochastic adversarial strategy for loss assignment  such that one action is significantly better than the other, but after applying noise on both, they look identical to the learner. See Appendix \ref{app:thm:15} for the full proof.

\section{Bandit Models}\label{sec:5}

In this section we study bandit models, where the learner observes only the noisy feedback for the action selected. In our notation, the learner selects $I_t \sim q_t$ and observers  only the feedback $c_{I_t, t}$. 

\subsection{Bandit with Constant Noise Model}

\subsubsection{Algorithm}

Using our conclusion from the \textit{Full Information with Constant Noise} setting, we present algorithm that do not use the noise parameter $\epsilon$. Clearly, this establish upper bound for both settings: the \textit{known noise} setting and the \textit{unknown noise} setting.

\begin{theorem}\label{thm:16.5}
Let $\epsilon \in [0,1]$ and denote $p = \frac{1-\epsilon}{2}$. Then, running \textit{Exponential Weights Scheme} under the \textit{Bandit with Constant Noise} setting with the following loss estimate
\[
EST(i, c_{i,t}, \vec{q}_t, I_t) = \frac{1}{q_{i,t}}c_{i,t}=\hat{\ell}_{i,t}
\]
and $\eta = \epsilon \; \sqrt[]{\frac{\ln K}{T K}}$, guarantees
$$
Regret(T) \leq \frac{2}{\epsilon}\sqrt[]{T K \ln K}
$$
\end{theorem}
The proof of Theorem \ref{thm:16.5} is similar in spirit to the proof of Theorem \ref{thm:4}. For completeness, we include a proof in Appendix \ref{app:thm:16.5}.

\subsubsection{Impossibility result}

In this section we present a lower bound that matches our upper bound,  up to a constant factor.

\begin{theorem}\label{thm:17}
Consider the \textit{Bandit with Constant Noise} setting with noise parameter $\epsilon \in (0,1)$. Then, for any learner algorithm there exists a sequence of loss vectors $\vec{\ell}_1,...,\vec{\ell}_T$ such that 
\[
Regret(T) = \Omega(\min \{\frac{1}{\epsilon} \; \sqrt{T K}, T\})
\]
\end{theorem}

The proof of the above theorem follows the methodology for lower bounds for \textit{multi-arm bandit} problems, we follow here the methodology proposed in \cite{slivkins2017introduction} and adapt it to our special setting. 

The idea is to define a stochastic strategy for loss assignments such that the \textit{learner} will have high expected regret, which implies that there exists a realization of a loss sequence such that on this loss sequence the \textit{learner} has high regret. Full proof is given in Appendix \ref{app:thm:17}.

\subsection{Bandit with Variable Noise Model}

In this section we investigate the \textit{Bandit with Variable Noise} settings. We concentrate on the case where the marginal distribution of $D$ for each action $i$ is the uniform distribution on $[0,1]$.

%As mention before, in this settings, at the beginning of each round $t$ a noise parameter $\epsilon_{i,t}$ is drawn from $D$ independently for every action $i$. After drawing an action $I_t$ the learner observes the \textit{$\epsilon_{I_t,t}$-noisy feedback} $c_{i,t}$ of action $I_t$. As before we denote also $p_{i,t} = \frac{1-\epsilon_{i,t}}{2}$.

\subsubsection{Algorithm}
We use the same idea as in the \textit{Full Information settings} and ignore ``too-noisy'' rounds where the noise is close to $\frac{1}{2}$. More formally, we will run the \textit{Exponential Weights Scheme} with the following estimator:
\[
EST(i, c_{i,t}, \vec{q}_t, I_t) = \frac{1}{q_{i,t}}\frac{c_{i,t}-p}{1-2p} \mathbbm{1}_{\{p_{i,t} \leq \frac{1-\theta}{2}\}} \mathbbm{1}_{\{I_t = i\}} = \hat{\ell}_{i,t} \;,
\]
where $\theta$ is a parameter.
We call the algorithm resulting from using the above estimator in the \textit{Exponential Weights Scheme} as the \textbf{Exp3-Threshold}. The following theorem (proof given in Appendix \ref{app:thm:18}) bounds the regret of the algorithm.
\begin{theorem}\label{thm:18}
Let $D$ be the noise distribution, such that for each action $i$ the marginal distribution $\epsilon_{i,t}$ is distributed $U(0,1)$ (but not necessarily independent for different actions). The \textbf{Exp3-Threshold} algorithm with the parameters
%\textit{Bandit with Variable Noise} setting with 
\[
\eta = \frac{(\ln K)^{2/3}}{K^{1/3}T^{2/3}} \text{ and } \theta = \frac{K^{1/3}(\ln K)^{1/3}}{T^{1/3}}
\]
has, in the \textit{Bandit with Variable Noise}, regret of at most
$$
Regret(T) \leq 3 T^{2/3}K^{1/3} (ln K)^{1/3}
$$
\end{theorem}

\subsubsection{Impossibility result}

We show a lower bound of $\Omega((TK)^{2/3})$. The proof (in Appendix \ref{app:thm:19}) is similar to the proof  of Theorem \ref{thm:10} for the \textit{Full Information} settings. 
\begin{theorem} \label{thm:19}
For any algorithm in the \textit{Bandit with Variable Noise} setting with $D=U(0,1)$ as the noise parameters distribution, there exist a series of loss vectors $\vec{\ell}_1,...,\vec{\ell}_T$ such that 
\[
Regret(T) = \Omega(T^{2/3}K^{1/3})
\]
\end{theorem}

%The proof of the theorem (given in Appendix \ref{app:thm:19}) follow the same pattern of the proof of Theorem \ref{thm:10}. 

\subsection{Importance of knowing the Noise}
In Theorem \ref{thm:15} we showed that in the \textit{Full Information with Variable Noise} setting, a learner cannot guarantee a sub-linear regret bound without observing the noise drawn for each action $p_{i,t}$ at each round $t$. Since in the \textit{Bandit with Variable Noise} setting the feedback is a restriction of the feedback in the \textit{Full Information with Variable Noise} setting, the same lower bound still holds, as stated in the following corollary. 

\begin{corollary}
Fix an algorithm for the  \textbf{Bandit with Variable Noise} model under the uniform marginal distribution and assume that in each round $t$ the learner does not observes the noise parameters $(\epsilon_{1,t},...,\epsilon_{K,t})$ before picking an action $I_t$, then there exist a sequence of loss vectors $\vec{\ell}_1,...,\vec{\ell}_T$ such that 
\[
Regret(T) = \Omega(T)
\]
\end{corollary}

\section{Discussion}
In this paper we investigated adversarial online learning problems where the feedback is corrupted by random noise. We presented and study different noise systems that apply to the \textit{full information} feedback and the \textit{bandit} feedback. We provided efficient algorithms, as well as upper and lower bounds on the regret.

This work can be extended in many ways. In our settings we apply the noise system on the classic \textit{full information} and \textit{bandit}. Similar noise system can be applied on intermediate models such as the one proposed by \citet{mannor2011bandits,AlonCGMMS17}. A different corrupting settings can be consider too. For example, a settings in which an adversary is corrupting the feedbacks under some restrictions.

%{\noindent \em Remainder omitted in this sample. See http://www.jmlr.org/papers/ for full paper.}

% Acknowledgements should go at the end, before appendices and references

\acks{This work was supported in part by a grant
from the Israel Science Foundation (ISF) and by the Tel Aviv University Yandex Initiative in Machine Learning.
AR would like to thank Dor Elboim for fruitful discussions on probability.}

\newpage
%\vskip 0.2in
\bibliography{cites.bib}

\begin{thebibliography}{19}
\providecommand{\natexlab}[1]{#1}
\providecommand{\url}[1]{\texttt{#1}}
\expandafter\ifx\csname urlstyle\endcsname\relax
  \providecommand{\doi}[1]{doi: #1}\else
  \providecommand{\doi}{doi: \begingroup \urlstyle{rm}\Url}\fi

\bibitem[Alon et~al.(2017)Alon, Cesa{-}Bianchi, Gentile, Mannor, Mansour, and
  Shamir]{AlonCGMMS17}
Noga Alon, Nicol{\`{o}} Cesa{-}Bianchi, Claudio Gentile, Shie Mannor, Yishay
  Mansour, and Ohad Shamir.
\newblock Nonstochastic multi-armed bandits with graph-structured feedback.
\newblock \emph{{SIAM} J. Comput.}, 46\penalty0 (6):\penalty0 1785--1826, 2017.

\bibitem[Angluin and Laird(1988)]{AngluinL:1988}
Dana Angluin and Philip Laird.
\newblock Learning from noisy examples.
\newblock \emph{Mach. Learn.}, 2\penalty0 (4):\penalty0 343--370, April 1988.

\bibitem[Audibert and Bubeck(2009)]{audibert2009minimax}
Jean-Yves Audibert and S{\'e}bastien Bubeck.
\newblock Minimax policies for adversarial and stochastic bandits.
\newblock In \emph{COLT}, pages 217--226, 2009.

\bibitem[Auer et~al.(2002)Auer, Cesa-Bianchi, Freund, and
  Schapire]{auer2002nonstochastic}
Peter Auer, Nicolo Cesa-Bianchi, Yoav Freund, and Robert~E Schapire.
\newblock The nonstochastic multiarmed bandit problem.
\newblock \emph{SIAM journal on computing}, 32\penalty0 (1):\penalty0 48--77,
  2002.

\bibitem[Bubeck and Cesa{-}Bianchi(2012)]{BubeckC12}
S{\'{e}}bastien Bubeck and Nicol{\`{o}} Cesa{-}Bianchi.
\newblock Regret analysis of stochastic and nonstochastic multi-armed bandit
  problems.
\newblock \emph{Foundations and Trends in Machine Learning}, 5\penalty0
  (1):\penalty0 1--122, 2012.

\bibitem[Cesa-Bianchi and Lugosi(2006)]{cesa2006prediction}
Nicolo Cesa-Bianchi and G{\'a}bor Lugosi.
\newblock \emph{Prediction, learning, and games}.
\newblock Cambridge university press, 2006.

\bibitem[Freund and Schapire(1997)]{freund1997decision}
Yoav Freund and Robert~E Schapire.
\newblock A decision-theoretic generalization of on-line learning and an
  application to boosting.
\newblock \emph{Journal of computer and system sciences}, 55\penalty0
  (1):\penalty0 119--139, 1997.

\bibitem[Gajane et~al.(2018)Gajane, Urvoy, and Kaufmann]{gajane2018corrupt}
Pratik Gajane, Tanguy Urvoy, and Emilie Kaufmann.
\newblock Corrupt bandits for preserving local privacy.
\newblock In \emph{ALT 2018-Algorithmic Learning Theory}, 2018.

\bibitem[Kalai and Vempala(2005)]{kalai2005efficient}
Adam Kalai and Santosh Vempala.
\newblock Efficient algorithms for online decision problems.
\newblock \emph{Journal of Computer and System Sciences}, 71\penalty0
  (3):\penalty0 291--307, 2005.

\bibitem[Kearns and Li(1993)]{KearnsL93}
Michael~J. Kearns and Ming Li.
\newblock Learning in the presence of malicious errors.
\newblock \emph{SIAM J. Comput.}, 22\penalty0 (4):\penalty0 807--837, 1993.

\bibitem[Klein and Young(1999)]{klein1999number}
Philip Klein and Neal Young.
\newblock On the number of iterations for dantzig-wolfe optimization and
  packing-covering approximation algorithms.
\newblock In \emph{International Conference on Integer Programming and
  Combinatorial Optimization}, pages 320--327. Springer, 1999.

\bibitem[Koc{\'a}k et~al.(2016)Koc{\'a}k, Neu, and Valko]{kocak2016online}
Tom{\'a}s Koc{\'a}k, Gergely Neu, and Michal Valko.
\newblock Online learning with noisy side observations.
\newblock In \emph{AISTATS}, pages 1186--1194, 2016.

\bibitem[Littlestone and Warmuth(1994)]{littlestone1994weighted}
Nick Littlestone and Manfred~K Warmuth.
\newblock The weighted majority algorithm.
\newblock \emph{Information and computation}, 108\penalty0 (2):\penalty0
  212--261, 1994.

\bibitem[Mannor and Shamir(2011)]{mannor2011bandits}
Shie Mannor and Ohad Shamir.
\newblock From bandits to experts: On the value of side-observations.
\newblock In \emph{Advances in Neural Information Processing Systems}, pages
  684--692, 2011.

\bibitem[Slivkins(2017)]{slivkins2017introduction}
Aleksandrs Slivkins.
\newblock Introduction to multi-armed bandits, 2017.

\bibitem[Valiant(1985)]{Valiant:1985}
L.~G. Valiant.
\newblock Learning disjunction of conjunctions.
\newblock In \emph{Proceedings of the 9th International Joint Conference on
  Artificial Intelligence - Volume 1}, IJCAI'85, pages 560--566, 1985.

\bibitem[Weissman and Merhav(2000)]{weissman2000universal}
T~Weissman and N~Merhav.
\newblock Universal prediction of binary individual sequences in the presence
  of noise. accepted to ieee trans.
\newblock \emph{Inform. Theory, September}, 2000.

\bibitem[Weissman et~al.(2001)Weissman, Merhav, and
  Somekh-Baruch]{weissman2001twofold}
Tsachy Weissman, Neri Merhav, and Anelia Somekh-Baruch.
\newblock Twofold universal prediction schemes for achieving the finite-state
  predictability of a noisy individual binary sequence.
\newblock \emph{IEEE Transactions on Information Theory}, 47\penalty0
  (5):\penalty0 1849--1866, 2001.

\bibitem[Wu et~al.(2015)Wu, Gy{\"o}rgy, and Szepesv{\'a}ri]{wu2015online}
Yifan Wu, Andr{\'a}s Gy{\"o}rgy, and Csaba Szepesv{\'a}ri.
\newblock Online learning with gaussian payoffs and side observations.
\newblock In \emph{Advances in Neural Information Processing Systems}, pages
  1360--1368, 2015.

\end{thebibliography}

\newpage
\appendix

\section{Proof of Lemma \ref{lem:3}}\label{app:lem:3}
The proof follows the standard analysis of exponential weighting schemes: let $W_t = \sum_{i=1}^K w_{i,t}$ using the algorithm update we can write 
\begin{equation*}
\begin{aligned}
\frac{W_{t+1}}{W_t} &= \sum_{i=1}^K \frac{w_{i,t+1}}{W_t} = 
\sum_{i=1}^K \frac{w_{i,t}e^{-\eta\hat{\ell}_{i,t}}}{W_t} = 
\sum_{i=1}^K q_{i,t}e^{-\eta\hat{\ell}_{i,t}} \\
&\leq  \sum_{i=1}^K q_{i,t} (1 - \eta\hat{\ell}_{i,t} + \eta^2(\hat{\ell}_{i,t})^2)
\qquad (\textrm{using  } e^x \leq 1 + x + x^2 \textrm{  for  } x \leq 1) \\
& = 1 - \eta \sum_{i=1}^K q_{i,t}\hat{\ell}_{i,t} + \eta^2 \sum_{i=1}^K q_{i,t} (\hat{\ell}_{i,t})^2
\end{aligned}
\end{equation*}
Taking logs and using $ln(1-x) \leq -x \textrm{ for all } x$ and summing for $t=1,2,...,T$ yields
\begin{equation*}
ln\frac{W_{T+1}}{W_1} \leq -\eta \sum_{t=1}^T \sum_{i=1}^K q_{i,t}\hat{\ell}_{i,t} + \eta^2 \sum_{t=1}^T \sum_{i=1}^K q_{i,t} (\hat{\ell}_{i,t})^2
\end{equation*}
Moreover, for any fixed action $k$ we have $W_t \geq w_{k,t}$, thus:
\begin{equation*}
\ln\frac{W_{T+1}}{W_1} \geq \ln \frac{w_{k,T+1}}{W_1} = -\eta \sum_{t=1}^T \hat{\ell}_{k,t} - \ln K
\end{equation*}
Putting together and rearranging gives:
\begin{equation*}
\sum_{t=1}^T \sum_{i=1}^K q_{i,t}\hat{\ell}_{i,t} -  \sum_{t=1}^T \hat{\ell}_{k,t} \leq \frac{\ln K}{\eta} +
\eta\sum_{t=1}^T \sum_{i=1}^K q_{i,t} (\hat{\ell}_{i,t})^2
\end{equation*}
\QEDB

\section{Proof of Theorem \ref{thm:2}} \label{app:thm:2}
Since
\[
\hat{\ell}_{i,t} = \frac{c_{i,t}-p}{1-2p} \in \{\frac{1-p}{1-2p},\frac{-p}{1-2p}\}=\{-\frac{1-\epsilon}{2\epsilon},\frac{1+\epsilon}{2\epsilon}\},
\]
we have that
\[
-\eta \hat{\ell}_{i,t} \leq \epsilon \sqrt{\frac{\ln K}{T}} \frac{1-\epsilon}{2\epsilon} \leq \frac{\sqrt{\frac{\ln K}{T}}}{2} \leq 1
\]
where the last equation uses $T \geq \frac{1}{4} \ln K$. Thus, we can apply Lemma \ref{lem:3} and obtain
\begin{equation*}
\sum_{t=1}^T \sum_{i=1}^K q_{i,t}\hat{\ell}_{i,t} -  \sum_{t=1}^T \hat{\ell}_{k,t} \leq \frac{\ln K}{\eta} +
\eta\sum_{t=1}^T \sum_{i=1}^K q_{i,t} (\hat{\ell}_{i,t})^2
\end{equation*} \\
Taking expectation on both sides and using that the estimator is unbiased
(i.e., $\E[\hat{\ell}_{i,t}] = \ell_{i,t}$)
yields,
\begin{equation*}
\sum_{t=1}^T \sum_{i=1}^K q_{i,t}\ell_{i,t} -  \sum_{t=1}^T \ell_{k,t} \leq 
 \frac{\ln K}{\eta} +\eta\sum_{t=1}^T \sum_{i=1}^K q_{i,t} \E[(\hat{\ell}_{i,t})^2]
\end{equation*}
Using the fact that $Regret(T) =  \sum_{t=1}^T \sum_{i=1}^K q_{i,t}\ell_{i,t} -  \min_{k\in A}\sum_{t=1}^T \ell_{k,t}$, 
we have,
\begin{equation*}
Regret(T) \leq \frac{\ln K}{\eta} +\eta\sum_{t=1}^T \sum_{i=1}^K q_{i,t} \E[(\hat{\ell}_{i,t})^2]
\end{equation*}
We bound the second moments of the estimate as follows,
\begin{equation*}
\E[(\hat{\ell}_{i,t})^2] = p\frac{(\bar{\ell}_{i,t} - p)^2}{(1-2p)^2} + (1-p)\frac{(\ell_{i,t} - p)^2}{(1-2p)^2} \leq \frac{1}{(1-2p)^2} = \frac{1}{\epsilon^2}\;,
\end{equation*}
where $\bar{\ell}_{i,t} = 1 - \ell_{i,t}$. Putting it back together and plugging $\eta = \epsilon \; \sqrt[]{\frac{\ln K}{T}}$ we obtain
\begin{equation*}
Regret(T) \leq \frac{\ln K}{\eta} +\frac{\eta T}{\epsilon^2} \leq \frac{2}{\epsilon} \; \sqrt[]{T \ln K}
\end{equation*}
\QEDB

\section{Proof of Theorem \ref{thm:4}} \label{app:thm:4}
By applying Lemma \ref{lem:3} and taking expectation on both sides we obtain
\begin{equation*}
\sum_{t=1}^T \sum_{i=1}^K q_{i,t} \E[\hat{\ell}_{i,t}] -  \sum_{t=1}^T \E[\hat{\ell}_{k,t}] \leq \frac{\ln K}{\eta} +
\eta\sum_{t=1}^T \sum_{i=1}^K q_{i,t}\E[(\hat{\ell}_{i,t})^2]
\end{equation*}
Calculating the expectation of the estimator $\hat{\ell}_{i,t}$, and since $\ell_{i,t}\in\{0,1\}$, we have,
%terms on the right hand side (\textit{RHS}) we have
\begin{equation*}
\E[\hat{\ell}_{i,t}] = (1-p)\ell_{i,t} + p \bar{\ell}_{i,t} = (1-2p)\ell_{i,t}+p=|\ell_{i,t} - p|
\end{equation*}
%For the lest hand side (\textit{LHS} from now), note that 
For the second moment we have $(\hat{\ell}_{i,t})^2 = c_{i,t}^2 = c_{i,t} \leq 1$. Putting things together we have
\begin{equation}
\label{eq:basic1}
\sum_{t=1}^T \sum_{i=1}^K q_{i,t}|\ell_{i,t} - p| -  \sum_{t=1}^T |\ell_{k,t} - p|
\leq \frac{\ln K}{\eta} + \eta T
\end{equation}
Using the notation of $\hat{L}_{ON, T}=\sum_{t=1}^T\sum_{i=1}^K q_{i,t}\hat{\ell}_{i,t}$, $\hat{L}_{k, T}=\sum_{t=1}^T\hat{\ell}_{k,t}$, $L_{ON, T}=\sum_{t=1}^T\sum_{i=1}^K q_{i,t}\ell_{i,t}$, and $L_{k, T}=\sum_{t=1}^T\ell_{k,t}$, we can write inequality (\ref{eq:basic1}) as
\begin{equation*}
\E[\hat{L}_{ON, T}] - \E[\hat{L}_{k,T}]
\leq \frac{\ln K}{\eta} + \eta T
\end{equation*}
Denote by $G_{t,b} = \{i \in A \;|\; \ell_{i,t} = b \}$ the set of actions with loss $b\in\{0,1\}$ in round $t$.
Denote by $Q_t = \sum_{i \in G_{t,1}}q_{i,t}$ the distribution mass the learner gives actions in $G_{t,1}$. Using this notation we have $L_{ON, T} = \sum_{t=1}^T Q_t$. 
Now calculate the value of the estimated losses of the online algorithm,
\begin{equation*}
\begin{aligned}
\E[\hat{L}_{ON, T}] &= \sum_{t=1}^T \sum_{i=1}^K q_{i,t}|\ell_{i,t} - p| = \sum_{t=1}^T [p\sum_{i \in G_{t,0}} q_{i,t} + (1-p) \sum_{i \in G_{t,1}} q_{i,t} ] \\
&= \sum_{t=1}^T [p(1-Q_t) + (1-p) Q_t ] = \sum_{t=1}^T [p + (1-2p) Q_t ] \\
&= (1-2p)L_{ON, T} + pT
\end{aligned}
\end{equation*}
Similarly,  for the term  $\E[\hat{L}_{k,T}]$ we have,
\begin{equation*}
\begin{aligned}
\E[\hat{L}_{k, T}] &= \sum_{t=1}^T |\ell_{k,t} - p| = \sum_{t | \ell_{t,k} = 0} p + \sum_{t | \ell_{t,k} = 1} (1-p)\\
&=p(T - L_{k, T}) + (1-p)  L_{k, T} = (1-2p)L_{k, T} + p T
\end{aligned}
\end{equation*}
Putting all together,
\begin{equation*}
\E[\hat{L}_{ON, T}] - \E[\hat{L}_{k,T}] = (1-2p)L_{ON, T} + p T - [(1-2p)L_{k, T} + pT] = (1-2p)[L_{ON, T} - L_{k, T}] 
\end{equation*}
Dividing by both sides of inequity by $(1-2p)$ and using $\eta = \sqrt{\frac{\ln K}{T}}$ we obtain that
\begin{equation*}
Regret(T) = L_{ON, T} - \min_{k\in A} L_{k, T} \leq \frac{1}{1-2p}(\frac{\ln K}{\eta} + \eta T)= \frac{2}{\epsilon}\sqrt[]{T \ln K}
\end{equation*}
\QEDB

\section{Proof of Theorem \ref{thm:5}}\label{app:thm:5}

To prove the theorem we first define the following adversarial loss assignment strategy: 
\begin{itemize}
\item the adversary initially picks uniformly a \textit{best action} $i^\star$ ($\forall i \; Pr[i^\star = i] = \frac{1}{K}$)
\item at round $t$: the adversary draws losses for the actions from the following distributions:\begin{enumerate}
\item for $i^\star$:  $\ell_{i^\star, t} \sim B(\frac{1}{2}-\delta)$
\item for $i \neq i^\star$:  $\ell_{i,t} \sim B(\frac{1}{2})$
\end{enumerate} 
\end{itemize}
where $\delta = \min\{\frac{1}{6\epsilon}\sqrt{\frac{\ln K}{T}}, \frac{1}{2}\}$.
Now we calculate the distribution of the \textbf{$\epsilon$-noisy feedback} $c_{i,t}$. Starting with the best action we have 
\begin{equation*}
\begin{aligned}
Pr[c_{i^\star,t} = 1] &= Pr[\ell_{i^\star, t} = 1]Pr[R_\epsilon = 0] + Pr[\ell_{i^\star, t} = 0]Pr[R_\epsilon = 1] \\
&= (\frac{1}{2}-\delta)\frac{1+\epsilon}{2} + (\frac{1}{2} + \delta)\frac{1-\epsilon}{2} = \frac{1}{2} - \epsilon\delta
\end{aligned}
\end{equation*}
For $i \neq i^\star$ we have
\begin{equation*}
\begin{aligned}
Pr[c_{i,t} = 1] &= Pr[\ell_{i, t} = 1]Pr[R_\epsilon = 0] + Pr[\ell_{i, t} = 0]Pr[R_\epsilon = 1] \\
&= \frac{1}{2}\frac{1+\epsilon}{2} + \frac{1}{2}\frac{1-\epsilon}{2} = \frac{1}{2}
\end{aligned}
\end{equation*}
Thus, we have: $c_{i^\star,t} \sim B(\frac{1}{2} - \epsilon\delta)$ and $c_{i,t} \sim B(\frac{1}{2})$ for $i\neq i^\star$.

The following is a standard claim regarding the minimum of i.i.d binomial random variables.

\begin{lemma}\label{lem:6}
Let $X_1,...,X_{K-1}$ be i.i.d random variables with distribute $B(n,p)$ such that $p \in (\frac{1}{4}, \frac{1}{2})$,  $n \geq 2 \ln K$ and $K \geq e^{18}$. Then with probability of at least $\frac{1}{2}$ we have
\begin{equation*}
\min \{X_1,...,X_{K-1}\} \leq np - \sqrt{\frac{p}{9} n \ln K}
\end{equation*}
\end{lemma}

\begin{proof}
Denote by $Y=\min \{X_1,...,X_{K-1}\}$ then by interdependency we can write 
\begin{equation}\label{eq:6.1}
\begin{aligned}
Pr[Y \leq np - t] &= 1 - Pr[\forall i \in \{1,2,...,K-1\} X_i \geq np - t] \\
&= 1 - (Pr[X_1 \geq np - t])^{K-1}
\end{aligned}
\end{equation}\label{eq:6.2}
Now we want to bound $Pr[X_1 \geq np - t]$. Rearranging, and using \textit{Lemma 5.2} of \citet{klein1999number}, for $t \leq \frac{1}{2}pn$  we can bound
\begin{equation*}
\begin{aligned}
Pr[X_1 \geq np - t] &= Pr[X_1 - np \geq - t] = 1 - Pr[X_1 - np \leq - t] \\ 
&\leq 1 - \exp(-\frac{9t^2}{np})
= 1 - \frac{1}{K}
\end{aligned}
\end{equation*}
where in the last equation we take $t = \sqrt{\frac{p}{9} n \ln K} \leq \frac{1}{2} pn$.
Plugging it back in (\ref{eq:6.1}) we obtain
\begin{equation*}
Pr[Y \leq np - \sqrt{\frac{p}{9} n \ln K}] \geq 1 - (1-\frac{1}{K})^{K-1} \geq \frac{1}{2}
\end{equation*}
\end{proof}

Denoting by $C_{i,T} = \sum_{t=1}^T c_{i,t}$, the sum of the noisy feedback of action $i$. Note that this is binomial random variable. In addition, for $i^\star$ we have $C_{i^\star, T} \sim B(T, \frac{1}{2}-\epsilon\delta)$ and for $i \neq i^\star$ we have $ C_{i,T} \sim B(T, \frac{1}{2})$. By applying Lemma \ref{lem:6} on the noisy-feedbacks
we show the following corollary.

\begin{corollary}\label{cor:7}
With probability at least $\frac{1}{4}$ there exist action $j \neq i^\star$ such that $C_{j,T} < C_{i^\star, T}$.
\end{corollary}

\begin{proof}
Applying Lemma \ref{lem:6} on the $K-1$ actions with $c_{i,t} \sim B(\frac{1}{2})$ we obtain that with probability at least $\frac{1}{2}$ there exist action $j \neq i^\star$ such that
\begin{equation*}
C_{j,T} \leq \frac{T}{2} - \sqrt{\frac{p}{9}T \ln K} < \frac{T}{2} - \frac{1}{6}\sqrt{T \ln K},
\end{equation*}
where the second inequality uses $p > 1/4$.\\
For the best action $i^\star$ we have $\E[C_{i^\star, T}] = \frac{T}{2} - \epsilon\delta T$. 
\begin{itemize}
\item if $\delta = \frac{1}{6\epsilon}\sqrt{\frac{\ln K}{T}}$ we have
\begin{equation*}
\E[C_{i^\star, T}] = \frac{T}{2} - \frac{1}{6}\sqrt{T \ln K}
\end{equation*}
Using the fact that for binomial distribution, $B(n,q)$, the median is $\floor{nq}$ or $\ceil{nq}$ we have that with probability at least $\frac{1}{2}$ 
\begin{equation*}
C_{i^\star, T} \geq \frac{T}{2} - \frac{1}{6}\sqrt{T \ln K}
\end{equation*}
\item  if $\delta = \frac{1}{2}$ we have that the distribution for the \textbf{$\epsilon$-noisy feedback} of the \textit{best action}, $c_{i^\star, t}$ is $B(\frac{1-\epsilon}{2})$, therefore
\begin{equation*}
\E[C_{i^\star, T}] = \frac{T}{2} - \frac{\epsilon}{2}T
\end{equation*}
$\delta = \frac{1}{2}$ implies $\epsilon \leq \frac{1}{3}\sqrt{\frac{\ln K}{T}}$ (as $\delta = \min\{\frac{1}{6\epsilon}\sqrt{\frac{\ln K}{T}}, \frac{1}{2}\}$) thus,  
\begin{equation*}
\frac{\epsilon}{2}T \leq \frac{T}{2}\frac{1}{3}\sqrt{\frac{\ln K}{T}} = \frac{1}{6}\sqrt{T \ln K}
\end{equation*}
Therefore, we still have that with probability at least $\frac{1}{2}$ 
\begin{equation*}
C_{i^\star, T} \geq \frac{T}{2} - \frac{1}{6}\sqrt{T \ln K}
\end{equation*}
\end{itemize}
Putting things together we obtain that with probability at least $\frac{1}{4}$ we have
\begin{equation*}
C_{i^\star, T} > C_{j,T}
\end{equation*}
\end{proof} 

The following lemma states that the action that has smaller observed noisy-loss has a higher probability to be the \textit{best action}.

\begin{lemma}\label{lem:8}
%Let $j_1, j_2 \in A$ two actions, and let 
Let $C_{1,T},\ldots,C_{K,T}$ be a realization of the noisy-feedbacks, such that $C_{j_1,T} < C_{j_2,T}$, where $j_1, j_2 \in A$.
Then,
\begin{equation*}
\Pr[i^\star = j_1  \mid C_{1,T},\ldots,C_{K,T}] > \Pr[i^\star = j_2 \mid C_{1,T},\ldots,C_{K,T}]
\end{equation*}
\end{lemma}

\begin{proof}
Using Bayes' theorem we have for action $j \in A$ that
\begin{equation*}
\begin{aligned}
\Pr[i^\star = j  \mid C_{1,T},\ldots,C_{K,T}] &= \frac{\Pr[C_{1,T},\ldots,C_{K,T} \mid i^\star = j] \Pr[i^\star = j]}{\Pr[C_{1,T},\ldots,C_{K,T}]} \\&
= \frac{\Pr[C_{j,T} \mid i^\star = j](\frac{1}{2}^T)^{K-1} \frac{1}{K}}{\Pr[C_{1,T},\ldots,C_{K,T}]} = \frac{\Pr[C_{j,T} \mid i^\star = j]}{Z} \\ &= \frac{1}{Z}(\frac{1-\epsilon}{2})^{C_{j,T}}(\frac{1+\epsilon}{2})^{T-C_{j,T}}
\end{aligned}
\end{equation*}
where $Z = \frac{(\frac{1}{2}^T)^{K-1} \frac{1}{K}}{\Pr[C_{1,T},\ldots,C_{K,T}]}$ is a constant (not depend on $j$). Therefore, if $C_{j_1,T} < C_{j_2,T}$ then
\begin{equation*}
\Pr[i^\star = j_1  \mid C_{1,T},\ldots,C_{K,T}] > \Pr[i^\star = j_2 \mid C_{1,T},\ldots,C_{K,T}]
\end{equation*}
\end{proof}

Using the lemma we can show the following corollary.

\begin{corollary}\label{cor:9}
consider an algorithm for ``predicting the best action'' problem: that is, the algorithm input is a realization  $C_{1,T},\ldots,C_{K,T}$, i.e., for one action $i^\star$ we have $C_{i^\star, T} \sim B(T, \frac{1}{2}-\epsilon\delta)$ and for $j \neq i^\star$ we have $C_{j,T} \sim B(T, \frac{1}{2})$ and the output is an action $I_T$ - a prediction for which action is optimal. Then for any algorithm we have, 
\begin{equation*}
\Pr[I_T \neq i^\star] \geq \frac{1}{4}
\end{equation*}
where the probability is taken over the randomness of the algorithm, the losses, the noise and the draw of $i^\star$.
\end{corollary}

\begin{proof}
Lemma \ref{lem:8} implies that the optimal algorithm will predict 
\begin{equation*}
I_T = \argmin_{j \in A} \{C_{1,T},\ldots,C_{K,T}\}
\end{equation*}
From Corollary \ref{cor:7} we have that for the optimal algorithm
\begin{equation*}
\Pr[I_T \neq i^\star] \geq \frac{1}{4}
\end{equation*}
\end{proof}

Putting it all together we can now prove the theorem.

\begin{proof}\textbf{of Theorem \ref{thm:5}:}
For any round $t$ we would think of the algorithm as algorithm for ``predicting the best action'' problem. Using this we can think of $t$ as the the time horizon and by applying Corollary \ref{cor:9} conclude that for every $t$ we have  
\begin{equation*}
\Pr[I_t \neq i^\star] \geq \frac{1}{4}
\end{equation*}
Therefore the expectation of the regret, when the expectation is taken over the losses, the noise and the draw of $i^\star$ (note that the regret itself includes the randomness of the algorithm) satisfies,
\begin{equation*}
\E[Regret(T)] = \sum_{t=1}^T \Pr[I_t \neq i^\star]\delta \geq \frac{1}{4}T\delta,
\end{equation*}
where $\delta = \min\{\frac{1}{6\epsilon}\sqrt{\frac{\ln K}{T}}, \frac{1}{2}\}$ concludes the proof. 
\end{proof}

\section{Proof of theorem \ref{thm:10}}\label{app:thm:10}
By applying Lemma \ref{lem:3} and taking expectation on both sides we obtain
\begin{equation}\label{eq:18}
\sum_{t=1}^T \sum_{i=1}^K q_{i,t} \E[\hat{\ell}_{i,t}] -  \sum_{t=1}^T \E[\hat{\ell}_{k,t}] \leq \frac{\ln K}{\eta} + \eta\sum_{t=1}^T \sum_{i=1}^K q_{i,t}\E[(\hat{\ell}_{i,t})^2]
\end{equation}
Conditioning on $p_{i,t} \leq \frac{1-\theta}{2}$, the estimator $\hat{\ell}_{i,t}$ is biased, however, we can bound the deviation. Specifically,
\begin{equation*}
\E[\hat{\ell}_{i,t}] = \theta * 0 + (1-\theta)\E[\hat{\ell}_{i,t} \; | \; p_{i,t} \leq \frac{1-\theta}{2}]=(1-\theta) \ell_{i,t} 
\end{equation*}
This implies that 
\[
\ell_{i,t}-\theta \leq \E[\hat{\ell}_{i,t}]\leq \ell_{i,t}
\]
To bound the second moment we have
\begin{equation*}
\E[(\hat{\ell}_{i,t})^2] = \theta*0 + (1-\theta) \E[(\hat{\ell}_{i,t})^2 \; | \; p_{i,t} \leq \frac{1-\theta}{2}] \leq \E[(\hat{\ell}_{i,t})^2 \; | \; p_{i,t} \leq \frac{1-\theta}{2}]
\end{equation*}
We bound the conditional expectation above as follows,
\begin{equation*}
\begin{aligned}
\E[(\hat{\ell}_{i,t})^2 \; | \; p_{i,t} \leq \frac{1-\theta}{2}] &= p_{i,t}\frac{(\bar{\ell}_{i,t} - p_{i,t})^2}{(1-2p_{i,t})^2} + (1-p_{i,t})\frac{(\ell_{i,t} - p_{i,t})^2}{(1-2p_{i,t})^2} \\
&\leq \frac{1}{(1-2p_{i,t})^2} = \frac{1}{\epsilon_{i,t}^2}
\end{aligned}
\end{equation*}
Computing the expectation, given that the marginal is $U(0,1)$, we have,
\begin{equation*}
\begin{aligned}
\E[(\hat{\ell}_{i,t})^2] &\leq \E[(\hat{\ell}_{i,t})^2 \; | \; p_{i,t} \leq \frac{1-\theta}{2}] \leq \E_{\epsilon \sim U(0, 1)} \; [\frac{1}{\epsilon^2} \mathbbm{1}_{\epsilon \geq \theta} ] \\
&= \int_{\theta}^{1} \frac{1}{\epsilon^2} d\epsilon = -[\frac{1}{\epsilon}]_{\theta}^1=\frac{1}{\theta} - 1 \leq \frac{1}{\theta}
\end{aligned}
\end{equation*}
Bounding the expressions in inequality (\ref{eq:18}) we obtain
\begin{equation*}
\begin{aligned}
\sum_{t=1}^T \sum_{i=1}^K q_{i,t} \E[\hat{\ell}_{i,t}] -  \sum_{t=1}^T \E[\hat{\ell}_{k,t}] 
&  \geq \sum_{t=1}^T \sum_{i=1}^K q_{i,t}\ell_{i,t} - \sum_{t=1}^T \ell_{k,t} - \theta T \\
 \frac{\ln K}{\eta} + \eta\sum_{t=1}^T \sum_{i=1}^K q_{i,t}\E[(\hat{\ell}_{i,t})^2]
& \leq \frac{\ln K}{\eta} +
\eta\sum_{t=1}^T \sum_{i=1}^K q_{i,t} \frac{1}{\theta} = \frac{\ln K}{\eta} + \frac{\eta T}{\theta}
\end{aligned}
\end{equation*}
Rearranging the terms gives us,
\[
\sum_{t=1}^T \sum_{i=1}^K q_{i,t}\ell_{i,t} - \sum_{t=1}^T \ell_{k,t} \leq \frac{\ln K}{\eta} + \frac{\eta T}{\theta} + \theta T
\]
Substituting $\eta = (\frac{\ln K}{T})^{2/3} \text{ and } \theta = (\frac{\ln K}{T})^{1/3}$ concludes the proof. 
\QEDB

\section{Proof of Theorem~\ref{thm:11}}\label{app:thm:11} 
Let $\theta= (\frac{\ln K}{T})^{1/3}$.
Initially, the adversary choose an action $i^{\star}$ uniformly at random, and it will be the best action. Then, for each round $t$ after observing $\epsilon_t$, the adversary assigns losses as follow:
\begin{enumerate}
\item If $\epsilon_t  \geq \theta$ then $\ell_{i,t} = 0 $ for every action $i$.
\item Otherwise ($\epsilon_t< \theta$) the adversary draw a loss for each action as follows: for action $i^{\star}$ the loss is drawn from $B(\frac{1}{2}-\frac{1}{6})$ and for any other action $j\neq i^{\star}$ it is drawn from $B(\frac{1}{2})$.
\end{enumerate}
Denote by $T'$ the number of \textit{bad rounds}. Since $\E[T'] = \theta T$ and the fact that for Binomial distribution, $B(n,p)$, the median is $\floor{np}$ or $\ceil{np}$ we conclude that with probability at least $1/2$ we have  $T' \geq \theta T$. Condition on this event we assume that $T' = \theta T$ (if $T' > \theta T$ we take the first $\theta T$ rounds to be $T'$) we reduce the \textit{bad rounds} to the constant noise setting in the following way:\\
In the \textit{bad rounds} we have $\epsilon_t \sim U(0,\theta)$. If we assume that in the \textit{bad rounds} we have $\epsilon_t = \theta$, namely a constant noise, then we only reduced the noise in the model. We call the model with $\epsilon_t = \theta \text{ and } T = T'$ the \textit{reduced model}. Therefore, a lower bound for the regret in the \textit{reduced model} is also a lower bound for a model where $\epsilon_t \sim U(0,\theta)$.\\
Our \textit{reduced model} is the \textbf{Full Information with Constant Noise} model with $T = T'$ and $\epsilon = \theta$. Denote by $Regret(T',\theta)$ the regret in the \textbf{Full Information with Constant Noise} model with horizon $T'$ and noise parameter $\theta$. Now, we can apply Theorem \ref{thm:5} on the \textit{reduced model} and obtain that
\begin{equation*}
Regret(T',\theta) \geq \gamma \frac{1}{\theta}\sqrt{T' \ln K} 
\end{equation*}
where $\gamma > 0$ is a constant. Setting $T' = \theta T = T^{2/3}(\ln K)^{1/3}$ we obtain that 
\begin{equation*}
Regret(\theta T,\theta) \geq \frac{1}{\theta}\sqrt{\theta T \ln K} = \gamma T^{2/3}(\ln K)^{1/3}
\end{equation*}
\noindent
Putting it back in the original model yields, 
\begin{equation*}
Regret(T) \geq \Pr[T' \geq \theta T]Regret(\theta T, \theta) \geq \frac{\gamma}{2}T^{2/3}(\ln K)^{1/3}
\end{equation*}
(We note that the number $\frac{1}{6}$ in the distribution $B(\frac{1}{2}-\frac{1}{6})$ the adversary uses, comes from the $\delta=\frac{1}{6\epsilon}\sqrt{\frac{\ln K}{T}}$ we use in the proof of theorem \ref{thm:5} with $\epsilon = \theta \text{ and } T = T'$).
\QEDB

\section{Proof of Theorem~\ref{thm:12}}\label{app:thm:12} 

We apply Lemma \ref{lem:3}, and taking expectation on both sides, obtain,
\begin{equation}\label{eq:28}
\sum_{t=1}^T \sum_{i=1}^K q_{i,t} \E[\hat{\ell}_{i,t}] -  \sum_{t=1}^T \E[\hat{\ell}_{k,t}] \leq \frac{\ln K}{\eta} + \eta\sum_{t=1}^T \sum_{i=1}^K q_{i,t}\E[(\hat{\ell}_{i,t})^2]
\end{equation}
Conditioning on $p_{i,t} \leq \frac{1-\theta}{2}$, the estimator $\hat{\ell}_{i,t}$ is biased, and we have
\begin{equation*}
\E[\hat{\ell}_{i,t}] = F(\theta) * 0 + (1-F(\theta))\E[\hat{\ell}_{i,t} \; | \; p_{i,t} \leq \frac{1-\theta}{2}]=(1-F(\theta)) \ell_{i,t} 
\end{equation*}
This implies that 
\[
\ell_{i,t}-F(\theta) \leq \E[\hat{\ell}_{i,t}]\leq \ell_{i,t}
\]
To bound the second moment we have
\begin{equation*}
\E[(\hat{\ell}_{i,t})^2] = F(\theta)*0 + (1-F(\theta)) \E[(\hat{\ell}_{i,t})^2 \; | \; p_{i,t} \leq \frac{1-\theta}{2}] \leq \E[(\hat{\ell}_{i,t})^2 \; | \; p_{i,t} \leq \frac{1-\theta}{2}]
\end{equation*}
We bound the above conditional expectation as follows,
\begin{equation*}
\begin{aligned}
\E[(\hat{\ell}_{i,t})^2 \; | \; p_{i,t} \leq \frac{1-\theta}{2}] &= p_{i,t}\frac{(\bar{\ell}_{i,t} - p_{i,t})^2}{(1-2p_{i,t})^2} + (1-p_{i,t})\frac{(\ell_{i,t} - p_{i,t})^2}{(1-2p_{i,t})^2} \\
&\leq \frac{1}{(1-2p_{i,t})^2} = \frac{1}{\epsilon_{i,t}^2}
\end{aligned}
\end{equation*}
Bounding each side of inequality (\ref{eq:28}) we have
\begin{equation*}
\begin{aligned}
\sum_{t=1}^T \sum_{i=1}^K q_{i,t} \E[\hat{\ell}_{i,t}] -  \sum_{t=1}^T \E[\hat{\ell}_{k,t}] 
&  \geq \sum_{t=1}^T \sum_{i=1}^K q_{i,t}\ell_{i,t} - \sum_{t=1}^T \ell_{k,t} - F(\theta) T \\
 \frac{\ln K}{\eta} + \eta\sum_{t=1}^T \sum_{i=1}^K q_{i,t}\E[(\hat{\ell}_{i,t})^2]
& \leq \frac{\ln K}{\eta} + \eta\sum_{t=1}^T \sum_{i=1}^K q_{i,t}\E[\frac{1}{\epsilon^{2}}\mathbbm{1}_{\epsilon \geq \theta}]
\end{aligned}
\end{equation*}
Rearranging it all yield
\[
\sum_{t=1}^T \sum_{i=1}^K q_{i,t}\ell_{i,t} - \sum_{t=1}^T \ell_{k,t} \leq \frac{\ln K}{\eta} + \eta T g(\theta) + F(\theta) T
\]
\QEDB

\section{Proof of Corollary~\ref{cor:13}}\label{app:cor:13} 
Using Theorem \ref{thm:12} and the assumption we can write
\begin{equation*}
Regret(T) \leq 2 \sqrt{g(\theta)T \ln K} + \theta^\alpha T
\end{equation*}
since $g(\theta)=\E[\frac{1}{\epsilon^{2}}\mathbbm{1}_{\epsilon \geq \theta}] \leq \frac{1}{\theta^2}$, we have,
\begin{equation*}
Regret(T) \leq \frac{2}{\theta}\sqrt{T \ln K} + \theta^\alpha T
\end{equation*}
taking $\theta = (\frac{2}{\alpha})^{\frac{1}{1+\alpha}}(\frac{\ln K}{T})^{\frac{1}{2(1+\alpha)}}$ gives
\begin{equation*}
Regret = O(T^{\frac{2+\alpha}{2+2\alpha}}(\ln K)^{\frac{\alpha}{2(1+\alpha)}})
\end{equation*}
\QEDB

\section{Proof of Corollary~\ref{cor:14}}\label{app:cor:14}
Applying Theorem \ref{thm:12} gives
\begin{equation}\label{eq:111}
Regret(T) \leq 2\sqrt{g(\theta)T \ln K} + F(\theta)T
\end{equation}
To bound $g(\theta)$ we calculate
\begin{equation*}
\begin{aligned}
g(\theta) &= \E[\frac{1}{\epsilon^{2}}\mathbbm{1}_{\epsilon \geq \theta}] \int_{\theta}^{1} \frac{1}{\epsilon^2}\frac{\lambda}{1-e^{-\lambda}}e^{-\lambda \epsilon} d\epsilon \leq  \frac{\lambda}{1-e^{-\lambda}} \int_{\theta}^{1} \frac{1}{\epsilon^2} d\epsilon \\
&= \frac{\lambda}{1-e^{-\lambda}}(\frac{1}{\theta} - 1) \leq  \frac{\lambda}{1-e^{-\lambda}}\frac{1}{\theta} \leq \frac{\lambda}{\theta}
\end{aligned}
\end{equation*}
Bounding the second term we use the inequality $1-e^{-x} \leq x \text{ for } x > 0$ and obtain
\begin{equation*}
F(\theta) = \frac{\lambda}{1-e^{-\lambda}}(1-e^{-\lambda \theta}) \leq \lambda^2 \theta
\end{equation*}
Putting it back in (\ref{eq:111}) we have
\begin{equation*}
Regret(T) \leq 2\sqrt{\frac{1}{\theta} \ln K} + \lambda^2 \theta T
\end{equation*}
setting $\theta = \frac{1}{\lambda}(\frac{\ln K}{T})^{1/3}$ yields,
\[
Regret(T) \leq 3 \lambda T^{2/3}(\ln K)^{1/3}
\]
\QEDB

\section{Proof of Theorem~\ref{thm:15}}\label{app:thm:15}
Let the number of actions be $K = 2$. Assume that initially the adversary picks the best action uniformly (that is, with probability $\frac{1}{2}$ action 1 will be the best action and with probability $\frac{1}{2}$ action 2 will be the best action). Let $i^\star \in \{1,2\}$ be a random variable
denoting the best action and $j = 3 - i^\star$ denote the worse action.
On round $t$, after observing the noise parameters $p_{1,t}$ and $p_{2,t}$, the adversary selects the
losses as follow:
\begin{enumerate}
\item  For the best action, $i^\star$, the loss is drawn at every round independently from a Bernoulli
r.v. with parameter $1/4$, i.e., $\ell_{i^\star, t} \sim B(\frac{1}{4})$
\item For the worse action $j$: if $p_{j,t} < 1/4$ then the loss is $\ell_{j,t} = 0$, otherwise the loss is $\ell_{j,t} = 1$.
\end{enumerate}
For the learner, observing the feedback $c_{i,t}=\ell_{i,t} \oplus r_{i,t}$, the loss of each action is a Bernoulli random variable. We will show that both actions will have the same probability of $1$, namely
$3/8$, and therefore indistinguishable by the learner.\par
Now we calculate the expected value of the observed feedback, $c_{i,t}=\ell_{i,t} \oplus r_{i,t}$, for
each action in a single round. We note that this expectation is taken over the draw of $\epsilon_{i,t} \sim U(0,1)$, the draw $R_{i,t} \sim B(\frac{1-\epsilon_{i,t}}{2})$  and the draw of the losses $\ell_{i,t}$. We also note that if $\epsilon \sim  U(0, 1)$  then $p \sim U(0,\frac{1}{2})$.\par
The expected loss of best action, $\ell_{i^\star,t}$ is drawn independently from the noise parameter $\epsilon_{i^\star, t}$ and the Bernoulli noise $R_{i,t}$. Therefore, we have
\begin{equation*}
\begin{aligned}
\E[c_{i^\star,t}] &= \E_p[\E_R[\E_\ell[\ell_{i^\star,t} \oplus R_{i^\star,t} ] \mid p]] = \E_p[\E_R[\frac{1}{4}(1 \oplus R_{i^\star,t}) + \frac{3}{4}(0 \oplus R_{i^\star,t}) \mid p]] \\
&= \frac{1}{4}\E_p[p_{i^\star,t} \cdot 0 + (1 - p_{i^\star,t}) \cdot 1] + \frac{3}{4}\E_p[p_{i^\star,t} \cdot 1 + (1 - p_{i^\star,t}) \cdot 0] \\
&=\frac{1}{4} \cdot \frac{3}{4} + \frac{3}{4} \cdot \frac{1}{4} = \frac{3}{8}
\end{aligned}
\end{equation*}
For the worse action, action $j$, we have
\begin{equation*}
\begin{aligned}
\E[c_{j,t]} &= \E[\ell_{j,t} \oplus R_{j,t} ] = \frac{1}{2}\E[0 \oplus R_{j,t} \mid p_{j,t} < 1/4] + \frac{1}{2}\E[1 \oplus R_{j,t} \mid \frac{1}{4} \leq p_{j,t} < \frac{1}{2}] \\
&= \frac{1}{2}\E[p_{j,t} \mid p_{j,t} < \frac{1}{4}] + \frac{1}{2} \E[ 1 - p_{j,t} \mid \frac{1}{4} \leq p_{j,t} < \frac{1}{2}] \\
&= \frac{1}{2} \cdot \frac{1}{8} + \frac{1}{2}(1 - \frac{3}{8}) = \frac{3}{8}
\end{aligned}
\end{equation*}
This implies that the feedback of both the best and worse action is a Bernoulli random
variable with parameter $\frac{3}{8}$, i.e., $B(\frac{3}{8})$. This clearly implies that the learner cannot distinguish
between the two actions, and therefore, half the time it will select the worse action.
The best action has an expected loss of $\frac{T}{4}$ while the worse action has a loss of $\frac{T}{2}$. This implies that the expected regret would be at least $\frac{T}{8}$.
\QEDB

\section{Proof of Theorem \ref{thm:16.5}} \label{app:thm:16.5}
By applying Lemma \ref{lem:3} and taking expectation on both sides we obtain
\begin{equation*}
\sum_{t=1}^T \sum_{i=1}^K q_{i,t} \E[\hat{\ell}_{i,t}] -  \sum_{t=1}^T \E[\hat{\ell}_{k,t}] \leq \frac{\ln K}{\eta} +
\eta\sum_{t=1}^T \sum_{i=1}^K q_{i,t}\E[(\hat{\ell}_{i,t})^2]
\end{equation*}
Calculating the expectation of the estimator $\hat{\ell}_{i,t}$, and since $\ell_{i,t}\in\{0,1\}$, we have,
%terms on the right hand side (\textit{RHS}) we have
\begin{equation*}
\E[\hat{\ell}_{i,t}]= q_{i,t}\frac{1}{q_{i,t}}\E[c_{i,t}] = \E[c_{i,t}] = (1-p)\ell_{i,t} + p \bar{\ell}_{i,t} = (1-2p)\ell_{i,t}+p=|\ell_{i,t} - p|
\end{equation*}
%For the lest hand side (\textit{LHS} from now), note that 
For the second moment, since $c_{i,t} \leq 1$ we have 
\begin{equation*}
\E[(\ell_{i,t})^2] = q_{i,t}\frac{1}{q_{i,t}^2}\E[c_{i,t}] \leq \frac{1}{q_{i,t}}
\end{equation*}
Putting things together we have
\begin{equation}
\label{eq:basic}
\sum_{t=1}^T \sum_{i=1}^K q_{i,t}|\ell_{i,t} - p| -  \sum_{t=1}^T |\ell_{k,t} - p|
\leq \frac{\ln K}{\eta} + \eta T K
\end{equation}
Using the notation of $\hat{L}_{ON, T}=\sum_{t=1}^T\sum_{i=1}^K q_{i,t}\hat{\ell}_{i,t}$, $\hat{L}_{k, T}=\sum_{t=1}^T\hat{\ell}_{k,t}$, $L_{ON, T}=\sum_{t=1}^T\sum_{i=1}^K q_{i,t}\ell_{i,t}$, and $L_{k, T}=\sum_{t=1}^T\ell_{k,t}$, we can write inequality (\ref{eq:basic}) as
\begin{equation*}
\E[\hat{L}_{ON, T}] - \E[\hat{L}_{k,T}]
\leq \frac{\ln K}{\eta} + \eta T K
\end{equation*}
Denote by $G_{t,b} = \{i \in A \;|\; \ell_{i,t} = b \}$ the set of actions with loss $b\in\{0,1\}$ in round $t$.
Denote by $Q_t = \sum_{i \in G_{t,1}}q_{i,t}$ the distribution mass the learner gives actions in $G_{t,1}$. Using this notation we have $L_{ON, T} = \sum_{t=1}^T Q_t$. 
Now calculate the value of the estimated losses of the online algorithm,
\begin{equation*}
\begin{aligned}
\E[\hat{L}_{ON, T}] &= \sum_{t=1}^T \sum_{i=1}^K q_{i,t}|\ell_{i,t} - p| = \sum_{t=1}^T [p\sum_{i \in G_{t,0}} q_{i,t} + (1-p) \sum_{i \in G_{t,1}} q_{i,t} ] \\
&= \sum_{t=1}^T [p(1-Q_t) + (1-p) Q_t ] = \sum_{t=1}^T [p + (1-2p) Q_t ] \\
&= (1-2p)L_{ON, T} + pT
\end{aligned}
\end{equation*}
Similarly,  for the term  $\E[\hat{L}_{k,T}]$ we have,
\begin{equation*}
\begin{aligned}
\E[\hat{L}_{k, T}] &= \sum_{t=1}^T |\ell_{k,t} - p| = \sum_{t | \ell_{t,k} = 0} p + \sum_{t | \ell_{t,k} = 1} (1-p)\\
&=p(T - L_{k, T}) + (1-p)  L_{k, T} = (1-2p)L_{k, T} + p T
\end{aligned}
\end{equation*}
Putting all together,
\begin{equation*}
\E[\hat{L}_{ON, T}] - \E[\hat{L}_{k,T}] = (1-2p)L_{ON, T} + p T - [(1-2p)L_{k, T} + pT] = (1-2p)[L_{ON, T} - L_{k, T}] 
\end{equation*}
Dividing by both sides of inequity by $(1-2p)$ and using $\eta = \sqrt{\frac{\ln K}{T K}}$ we obtain that
\begin{equation*}
Regret(T) = L_{ON, T} - \min_{k\in A} L_{k, T} \leq \frac{1}{1-2p}(\frac{\ln K}{\eta} + \eta T K)= \frac{2}{\epsilon}\sqrt[]{T K\ln K}
\end{equation*}
\QEDB

\section{Proof of Theorem~\ref{thm:17}}\label{app:thm:17} 

We first define $K$ different problem instances, one per action. Let $\beta \in (0,1)$ be a parameter. We denote by $J_i$ the problem instance where action $i$ loss is drawn from the distribution $B(\frac{1-\beta}{2})$ while the other actions loss is drawn from the distribution $B(\frac{1}{2})$. For problem instance $J_i$, we refer action $i$ as the \textit{best action}. The proof will show that in some sense those instances are indistinguishable for any algorithm. 

For the proof, we will think of the online algorithm as a leaner making ``prediction'' for the best action at each round $t$. The main part of the proof is to show that if $T$ is not large enough the algorithm has to have a constant mistake rate. 

We denote by $\Pr[I_t=i | J_i]$ the probability that in instance $J_i$, at round $t$ the algorithm selects action $i$ (the best action in instance $J_i$).
%Now we start with lemma:
The following lemma shows that for many actions the algorithm will make a mistake.

\begin{lemma}\label{lem:20}
Consider a deterministic algorithm for the \textit{Bandit with Constant Noise} problem with noise $p=\frac{1-\epsilon}{2}$. There exist a constant $\gamma$ such that if $t < \gamma\frac{K}{\epsilon^2 \beta^2}$ then there exist at least $\lceil \frac{K}{2} \rceil$ actions $i$ such that
\begin{equation*}
\Pr[I_t=i | J_i] < \frac{3}{4}
\end{equation*}
\end{lemma}

\begin{proof}
Consider the feedback distribution for each problem instance $J_j$ and action $i$. First, if $\ell_{i,t} \sim B(\frac{1}{2})$ then $c_{i,t} \sim B(\frac{1}{2})$ (the noise does not have any influence). For the \textit{best action}, i.e., $j$, we have $c_{j,t} \sim B(\frac{1-\epsilon\beta}{2})$ since
\begin{equation*}
\begin{aligned}
Pr[c_{j,t}=1] &= Pr[\ell_{j,t}=1]Pr[R_\epsilon = 0] + Pr[\ell_{j,t}=0]Pr[R_\epsilon = 1] \\
&= (\frac{1-\beta}{2})(\frac{1+\epsilon}{2})+(\frac{1+\beta}{2})(\frac{1-\epsilon}{2}) = \frac{1 - \epsilon\beta}{2}
\end{aligned}
\end{equation*}
Applying \textit{Lemma 2.10} of \citet{slivkins2017introduction} on the feedbacks $c_{i,t}$ completes the proof. 
\end{proof}

\begin{corollary}\label{cor:21}
%Assume $t < \gamma\frac{K}{\epsilon^2 \beta^2}$ (as in Lemma \ref{lem:20}, with the same constant $\gamma$). 
Choose the best action $i^\star$ uniformly from $A$ and use instance $J_{i^\star}$. For any algorithm, for any round $t < \gamma\frac{K}{\epsilon^2 \beta^2}$, we have  $\Pr[I_t \neq i^\star]\geq 1/8$. 
\end{corollary}

\begin{proof}
For a deterministic algorithm the corollary follows since by Lemma \ref{lem:20} with probability at least $ \frac{1}{2}$ the selected $i^\star$ is such that $\Pr[I_t\neq i^\star | J_{i^\star}] \geq \frac{1}{4}$. 
Since a randomized algorithm is a distribution over deterministic algorithms that claim hold also for randomized algorithms.
\end{proof}

\begin{proof}\textbf{of Theorem \ref{thm:17}:}
Let $\beta = \min\{\frac{\sqrt{\gamma}}{\epsilon}\sqrt{\frac{K}{T}}, 1\}$.
By Corollary \ref{cor:21}, 
%let assume $T < c\frac{K}{\epsilon^2 \beta^2}$, then 
we have that in each round $t$
%[[YM: The notation for the selected action is $I_t$ up to now]]
\begin{equation*}
\Pr[I_t \neq i^\star ] \geq \frac{1}{8}
\end{equation*}
Denote by $\Delta_t = \E[\ell_{I_t, t}] - \E[\ell_{i^\star, t}]$ the regret of round $t$. Note that if $I_t \neq i^\star$ then $\Delta_t = \frac{1}{2} - \frac{1-\beta}{2} = \frac{\beta}{2}$. Therefore, the expected regret at round $t$ is
\begin{equation*}
\E[\Delta_t] = \Pr[I_t \neq i^\star] \frac{\beta}{2} 
\end{equation*}
Summing over the rounds we have, 
\begin{equation*}
Regret(T) = \sum_{t=1}^T\E[\Delta_t] \geq \frac{1}{16} \beta T 
\end{equation*}
Since $\beta = \min\{\frac{\sqrt{\gamma}}{\epsilon}\sqrt{\frac{K}{T}}, 1\}$, we have
\begin{equation*}
Regret(T) \geq \min \{\frac{\sqrt{\gamma}}{16} \frac{1}{\epsilon}\sqrt{TK}, \frac{1}{16}T \}
\end{equation*}
\end{proof}

\section{Proof of Theorem~\ref{thm:18}}\label{app:thm:18} 
By applying Lemma \ref{lem:3} and taking expectation on both sides we obtain
\begin{equation}\label{eq:18a}
\sum_{t=1}^T \sum_{i=1}^K q_{i,t} \E[\hat{\ell}_{i,t}] -  \sum_{t=1}^T \E[\hat{\ell}_{k,t}] \leq \frac{\ln K}{\eta} + \eta\sum_{t=1}^T \sum_{i=1}^K q_{i,t}\E[(\hat{\ell}_{i,t})^2]
\end{equation}
Conditioning on $p_{i,t} \leq \frac{1-\theta}{2}$, the estimator $\hat{\ell}_{i,t}$ is unbiased, since
\begin{equation*}
\E[\hat{\ell}_{i,t} \; \mid  \; p_t \leq \frac{1-\theta}{2}] =q_{i,t}[\frac{1}{q_{i,t}} \frac{p\bar{\ell}_{i,t} + (1-p)\ell_{i,t} - p}{1-2p}] = \ell_{i,t}\;.
\end{equation*}
However, overall the estimator is biased,
\begin{equation*}
\E[\hat{\ell}_{i,t}] = \theta * 0 + (1-\theta)\E[\hat{\ell}_{i,t} \; | \; p_{i,t} \leq \frac{1-\theta}{2}]=(1-\theta) \ell_{i,t} 
%\leq \ell_{i,t}
\end{equation*}
This implies that 
\[
\ell_{i,t}-\theta \leq \E[\hat{\ell}_{i,t}]\leq \ell_{i,t}
\]
To bound the second moment we have
\begin{equation*}
\E[(\hat{\ell}_{i,t})^2] = \theta*0 + (1-\theta) \E[(\hat{\ell}_{i,t})^2 \; | \; p_{i,t} \leq \frac{1-\theta}{2}] \leq \E[(\hat{\ell}_{i,t})^2 \; | \; p_{i,t} \leq \frac{1-\theta}{2}]
\end{equation*}
The conditional expectation of the second moment is bounded as follows,
\begin{equation*}
\E[(\hat{\ell}_{i,t})^2 \; | \; p_t \leq \frac{1-\delta}{2}] = \frac{1}{q_{i,t}} [ p_t\frac{(\bar{\ell}_{i,t} - p_t)^2}{(1-2p_t)^2} + (1-p_t)\frac{(\ell_{i,t} - p_t)^2}{(1-2p_t)^2}] \leq \frac{1}{q_{i,t}} \frac{1}{(1-2p_t)^2} = \frac{1}{q_{i,t}} \frac{1}{\epsilon_t^2}
\end{equation*}
Since the marginal of the noise distribution $D$ is uniform, we have,
\begin{equation}
\begin{aligned}
\E[(\hat{\ell}_{i,t})^2] &\leq \E[(\hat{\ell}_{i,t})^2 \; | \; p_{i,t} \leq \frac{1-\theta}{2}] \leq \E_{\epsilon \sim U(0, 1)} \; [\frac{1}{q_{i,t}}\frac{1}{\epsilon^2} \mathbbm{1}_{\epsilon \geq \theta} ] \\
&= \frac{1}{q_{i,t}} \int_{\theta}^{1} \frac{1}{\epsilon^2} d\epsilon = -\frac{1}{q_{i,t}}[\frac{1}{\epsilon}]_{\theta}^1= \frac{1}{q_{i,t}}(\frac{1}{\theta} - 1) \leq \frac{1}{q_{i,t}}\frac{1}{\theta}
\end{aligned}
\end{equation}
Bounding each side of inequality (\ref{eq:18a}) we have
\begin{equation}
\begin{aligned}
\sum_{t=1}^T \sum_{i=1}^K q_{i,t} \E[\hat{\ell}_{i,t}] -  \sum_{t=1}^T \E[\hat{\ell}_{k,t}] 
&  \geq \sum_{t=1}^T \sum_{i=1}^K q_{i,t}\ell_{i,t} - \sum_{t=1}^T \ell_{k,t} - \theta T \\
 \frac{\ln K}{\eta} + \eta\sum_{t=1}^T \sum_{i=1}^K q_{i,t}\E[(\hat{\ell}_{i,t})^2]
& \leq \frac{\ln K}{\eta} +
\eta\sum_{t=1}^T \sum_{i=1}^K q_{i,t}[\frac{1}{q_{i,t}}\frac{1}{\theta}] = \frac{\ln K}{\eta} + \frac{\eta T K}{\theta}
\end{aligned}
\end{equation}
Rearranging it all yield
\[
\sum_{t=1}^T \sum_{i=1}^K q_{i,t}\ell_{i,t} - \sum_{t=1}^T \ell_{k,t} \leq \frac{\ln K}{\eta} + \frac{\eta T K}{\theta} + \theta T
\]
Substituting $\eta = \frac{(ln K)^{2/3}}{K^{1/3}T^{2/3}} \text{ and } \theta = \frac{K^{1/3}(ln K)^{1/3}}{T^{1/3}}$ concludes the proof. 
\QEDB

\section{Proof of Theorem~\ref{thm:19}}\label{app:thm:19} 

Let $\theta= (\frac{K}{T})^{1/3}$.
Initially, the adversary choose an action $i^{\star}$ uniformly at random, and it will be the best action. Then, for each round $t$ after observing $\epsilon_t$, the adversary assigns losses as follow: fix $\beta = \frac{\sqrt{\gamma}}{\theta}\sqrt{\frac{K}{T}} = \sqrt{\gamma}(\frac{K}{T})^{1/6}$ and at round $t$ do
\begin{enumerate}
\item if $\epsilon_t  \geq \theta$ then $\ell_{i,t} = 0 $ for every action $i$.
\item Otherwise ($\epsilon_t< \theta$) the adversary draw a loss for each action as follows: for action $i^{\star}$ the loss is drawn from $B(\frac{1}{2}-\beta)$ and for any other action $j\neq i^{\star}$ it is drawn from $B(\frac{1}{2})$.
\end{enumerate}
Denote by $T'$ the number of \textit{bad rounds}. Since $\E[T'] = \theta T$ and the fact that for Binomial distribution, $B(n,p)$, the median is $\floor{np}$ or $\ceil{np}$ we conclude that with probability at least $1/2$ we have  $T' \geq \theta T$. Condition on this event we assume that $T' = \theta T$ (if $T' > \theta T$ we take the first $\theta T$ rounds to be $T'$) we reduce the \textit{bad rounds} to the constant noise setting in the following way:\\
In the \textit{bad rounds} we have $\epsilon_t \sim U(0,\theta)$. If we assume that in the \textit{bad rounds} we have $\epsilon_t = \theta$, namely a constant noise, then we only reduced the noise in the model. We call the model with $\epsilon_t = \theta \text{ and } T = T'$ the \textit{reduced model}. Therefore, a lower bound for the regret in the \textit{reduced model} is also a lower bound for a model where $\epsilon_t \sim U(0,\theta)$.\\
Our \textit{reduced model} is the \textbf{Bandit with Constant Noise} model with $T = T'$ and $\epsilon = \theta$. Denote by $Regret(T',\theta)$ the regret in the \textbf{Bandit with Constant Noise} model with horizon $T'$ and noise parameter $\theta$. Now, we can apply Theorem \ref{thm:15} on the \textit{reduced model} and obtain that
\begin{equation*}
Regret(T',\theta) \geq \gamma \frac{1}{\theta}\sqrt{T' K} 
\end{equation*}
where $\gamma > 0$ is a constant. Setting $T' = \theta T = T^{2/3}K^{1/3}$ we obtain that 
\begin{equation*}
Regret(\theta T,\theta) \geq \frac{1}{\theta}\sqrt{\theta T K} = \gamma T^{2/3}K^{1/3}
\end{equation*}
\noindent
Putting it back in the original model yields, 
\begin{equation*}
Regret(T) \geq \Pr[T' \geq \theta T]*Regret(\theta T, \theta) \geq \frac{\gamma}{2}T^{2/3} K^{1/3}
\end{equation*}
(We note here that the choice of $\beta$ is according to the proof of Theorem \ref{thm:15} with $\epsilon = \theta \text{ and } T = T' = \theta T$).
\QEDB

\end{document}